\documentclass[11pt]{article}
\usepackage[utf8]{inputenc}
\usepackage{algorithm}
\usepackage{algorithmic}
\usepackage{amsthm}
\usepackage{amssymb}
\usepackage{bbm}
\usepackage{amsmath}
\usepackage{graphicx}
\usepackage{xcolor}
\usepackage{hhline}
\usepackage{hyperref}
\usepackage{breakcites}

\advance \oddsidemargin by -0.8in
\newcommand{\myparskip}{3pt}

\parskip \myparskip

\usepackage{amsmath,amsfonts,amssymb}
\usepackage{mathtools}
\usepackage{amsthm} 
\usepackage{latexsym}
\usepackage{relsize}

\usepackage[capitalise]{cleveref}
\usepackage{xcolor}
\usepackage{dsfont}

\newcommand{\ignore}[1]{}

\newcommand{\email}[1]{\texttt{#1}}

\theoremstyle{plain}
\newtheorem{theorem}{Theorem}
\newtheorem{lemma}[theorem]{Lemma}

\newtheorem{proposition}[theorem]{Proposition}

\newtheorem*{theorem*}{Theorem}
\newtheorem*{lemma*}{Lemma}
\newtheorem*{corollary*}{Corollary}
\newtheorem*{proposition*}{Proposition}
\newtheorem*{claim*}{Claim}
\newtheorem*{fact*}{Fact}
\newtheorem*{observation*}{Observation}
\newtheorem*{assumption*}{Assumption}

\theoremstyle{definition}
\newtheorem{definition}[theorem]{Definition}

\newtheorem*{definition*}{Definition}
\newtheorem*{remark*}{Remark}
\newtheorem*{example*}{Example}

 \theoremstyle{plain}
\newtheorem*{theoremaux}{\theoremauxref}
\gdef\theoremauxref{1}

\DeclareMathAlphabet{\mathbfsf}{\encodingdefault}{\sfdefault}{bx}{n}

\let\oldtfrac\tfrac
\renewcommand{\tfrac}[2]{\smash{\oldtfrac{#1}{#2}}}

\let\nablaold\nabla
\renewcommand{\nabla}{\nablaold\mkern-2.5mu}

\title{On the Computational Benefit of Multimodal Learning}
\author{Zhou Lu\\ Princeton University\\ \texttt{zhoul@princeton.edu}}

\begin{document}

\maketitle

\begin{abstract}
Human perception inherently operates in a multimodal manner. Similarly, as machines interpret the empirical world, their learning processes ought to be multimodal. The recent, remarkable successes in empirical multimodal learning underscore the significance of understanding this paradigm. Yet, a solid theoretical foundation for multimodal learning has eluded the field for some time. While a recent study by \cite{zhoul} has shown the superior sample complexity of multimodal learning compared to its unimodal counterpart, another basic question remains: does multimodal learning also offer computational advantages over unimodal learning? This work initiates a study on the computational benefit of multimodal learning. We demonstrate that, under certain conditions, multimodal learning can outpace unimodal learning exponentially in terms of computation. Specifically, we present a learning task that is NP-hard for unimodal learning but is solvable in polynomial time by a multimodal algorithm. Our construction is based on a novel modification to the intersection of two half-spaces
problem.
\end{abstract}

\section{Introduction}
At the heart of human perception lies multimodality. This capability enables us to perceive and interrelate different facets of the same empirical object. It's particularly important during the infantile stage of human development, where it helps unify disparate symbols, fostering comprehensive cognition as a foundation for adulthood. The analogy of raising a child in a "room of text" alone highlights the limitations of a unimodal approach; it's bound to be counterproductive.

In the realm of machine learning, multimodality plays a role analogous to its significance in human cognition. Here, we view machine learning as the machine's process of perception. Multimodal learning entails accumulating vast amounts of training data across various modalities and subsequently deploying the trained model to handle new unimodal tasks. This learning progression mirrors the transition from infancy to adulthood in humans. Empirical studies have consistently shown that models trained using multiple modalities often surpass finely-tuned unimodal models, even when evaluated on new unimodal data.

In spite of notable empirical successes, like Gato \cite{reed2022generalist} and GPT-4 \cite{openai2023gpt}, the theoretical explanations of multimodal learning remain relatively underexplored. Thus, establishing a solid theoretical foundation becomes imperative.

A recent study by \cite{zhoul} set the stage for a broader understanding of the statistical advantages of multimodal learning. The research showed that multimodal learning achieves superior generalization bounds compared to unimodal learning, especially when the data exhibits both connection and heterogeneity. However, the question arose: does multimodal learning also present computational advantages?

Our work provides an affirmative answer. We show a computational separation between multimodal and unimodal learning. Specifically, we introduce a learning task, rooted in the intersection of two half-spaces problem, which poses an NP-hard challenge for any unimodal learning algorithm. Yet, this very task yields to a polynomial solution under a multimodal learning paradigm. This dichotomy demonstrates the potential exponential computational advantage of multimodal learning over its unimodal counterpart. Coupled with the statistical insights from \cite{zhoul}, our findings further illuminate the vast potential of multimodal learning.

\subsection{Related Works}
\textbf{Theoretical Multimodal Learning}: despite the empirical success of multimodal learning, a cohesive theoretical foundation was long missing in this area. Most existing theoretical findings are bound by specific assumptions and contexts. For instance, studies such as \cite{zhang2019cpm, amini2009learning, federici2020learning, sridharan2008information} navigate multimodal learning within a multi-view framework, operating under the assumption that individual modalities are, in isolation, adequate for predictions. \cite{sun2020tcgm,liang2023quantifying} delve into algorithms pivoting on information-theoretical relationships across modalities. \cite{ren2023importance} consider the specific problem of the benefit of contrastive loss in multimodal learning with a linear
data-generating model. \cite{huang2021makes} studies the generalization ability of multimodal learning in estimating the latent space representation.

A recent work \cite{zhoul} proposes a broad-based theory on the statistical guarantee of multimodal learning. They prove that multimodal learning admits an $O(\sqrt{m})$ improvement in generalization error over unimodal learning. This is achieved by dissecting the learning of the composition of two hypotheses, where the sum of complexities of the hypotheses is markedly smaller than that of their composition. Additionally, they pinpoint connection and heterogeneity amidst modalities as the two pivotal elements propelling these statistical advantages of multimodal learning.

\textbf{Empirical Multimodal Learning}: applications of multimodal learning can be traced back to the last century, aiming at combining vision and audio data to improve the performance of speech recognition \cite{yuhas1989integration, ngiam2011multimodal}. As the field evolved, multimodal learning carved a niche in multimedia, enhancing capabilities in indexing and search functionalities \cite{evangelopoulos2013multimodal, lienhart1998comparison}. 

Recently, there is a trend in applying multimodal learning in deep learning practices, including modality generation \cite{chang2015text,hodosh2013framing,reed2016generative} and large-scale generalist models \cite{reed2022generalist, openai2023gpt}. A consistently observed empirical phenomenon is that a multimodal model is able to outperform a finely-tuned unimodal model, even on unimodal population data.

\section{Setting}
In this section, we delineate the setup of multimodal learning and essential background on the intersection of two half-spaces problem.

\subsection{Multimodal Learning Setup}
In this paper, we restrict our focus to the fundamental, yet non-trivial, scenario of two modalities for a clear exposition, adopting the setup of \cite{zhoul}. Formally, the multimodal learning classification framework encompasses two modalities, denoted as $\mathcal{X},\mathcal{Y}\subset \mathbb{R}^n$, and a label space $\mathcal{Z}=\{\pm\}$. Consequently, every data point can be represented as a tuple $(x,y,z)$.

Given a hypothesis class $\mathcal{H}$ and a training dataset $(X,Y,Z)$ with $m$ data points $(x_i,y_i,z_i)$, our aim in (proper) learning from $(X,Y,Z)$ is to output a hypothesis $h\in \mathcal{H}$, that minimizes the empirical risk:
$$
\ell_{emp}=\frac{\sum_{i=1}^m \mathbf{1}_{h(x_i,y_i)\ne z_i}}{m}.
$$
When each data point $(x,y,z)$ adheres to a specific data distribution $D$ over $(\mathcal{X},\mathcal{Y},\mathcal{Z})$, the goal of (properly) PAC-learning $(X,Y,Z)$ is to output a hypothesis $h\in \mathcal{H}$, such that the population risk
$$
\ell_{pop}=\mathbb{E}_{(x,y,z)\sim D}[\mathbf{1}_{h(x,y)\ne z}]
$$
is small with high probability. In addition, we mandate a bijective mapping between $x,y$ for any potential data point $(x,y,z)$. 

For brevity, we occasionally write $(\mathcal{X},\mathcal{Y},\mathcal{Z})$ for short to denote the learning problem, when it is clear from the context. The unimodal learning problems $(\mathcal{X},\mathcal{Z})$ and $(\mathcal{Y},\mathcal{Z})$ can be defined in a similar way, in which the label $y$ or $x$ is masked. In learning $(\mathcal{X},\mathcal{Z})$, we are given a hypothesis class $\mathcal{H}$ and a set $(X,Z)$ of training data with $m$ data points $(x_i,z_i)$. The empirical risk and the population risk are defined as follows respectively.
$$
\ell_{emp}=\frac{\sum_{i=1}^m \mathbf{1}_{h(x_i)\ne z_i}}{m},\ \ \ell_{pop}=\mathbb{E}_{(x,y,z)\sim D}[\mathbf{1}_{h(x)\ne z}].
$$

\subsection{Intersection of Two Half-spaces}
In our quest to demonstrate a computational separation between multimodal and unimodal learning, we sought to architect a specific learning challenge that presents as NP-hard for unimodal learning, but for which an efficient multimodal solution exists.

A candidate of such problem is the 'intersection of two half-spaces,' formally defined below:

\begin{definition}[Intersection of two half-spaces]
    An instance of $\mathbf{IntHS}$ is a set of points in $\mathbb{R}^n$ each labeled either ‘+’ or ‘-’ and the goal is to find an intersection of two half-spaces which correctly classifies the maximum number of points, where a ‘+’ point is classified correctly if it lies inside the intersection and a ‘-’ point is classified correctly if it lies outside of it.
\end{definition}

Previous work has shown that PAC-learning this intersection is inherently NP-hard, even in the realizable setting, encapsulated in the following result:

\begin{theorem}[\cite{khot2008hardness}]\label{thm ks}
Let $\ell$ be any fixed integer and $\epsilon>0$ be an
arbitrarily small constant. Then, given a set of labeled points
in $\mathbb{R}^n$ with a guarantee that there is an intersection of two half-spaces that classifies all the points correctly, there is no polynomial time algorithm to find a function $f$ of up to $\ell$ linear threshold functions that classifies $\frac{1}{2}+\epsilon$ fraction of points correctly, unless NP = RP.
\end{theorem}

A slightly weaker version of the above result which will be of use is the following:

\begin{proposition}\label{prop ks}
Let $\epsilon>0$ be an arbitrarily small constant. Then, given a set of labeled points in $\mathbb{R}^n$ with a guarantee that there is an intersection of two half-spaces that classifies all the points correctly, there is no polynomial time algorithm to find a function $f$ of an intersection of two half-spaces that classifies $\frac{1}{2}+\epsilon$ fraction of points correctly, unless NP = RP.
\end{proposition}

It's clear Proposition \ref{prop ks} is a direct consequence of Theorem \ref{thm ks}, given that the intersection of two half-spaces naturally translates to $\ell$ linear threshold functions with $\ell=2$. Through out this paper we will only consider the case of proper learning with our hypothesis class including only intersections of two half-spaces.

\section{A Computational Separation between Multimodal and Unimodal Learning}
To demonstrate the computational benefit of multimodal learning, we present an instance in which both unimodal learning problems $(\mathcal{X},\mathcal{Z})$ and $(\mathcal{Y},\mathcal{Z})$ are NP-hard, while the multimodal learning problem $(\mathcal{X},\mathcal{Y},\mathcal{Z})$ can be solved efficiently. In particular, we require the existence of a bijective mapping $f:\mathcal{X}\to \mathcal{Y}$ satisfying $y=f(x)$ for any data point $(x,y,z)\in (\mathcal{X},\mathcal{Y},\mathcal{Z})$, so that the hardness result is purely computational. The task of constructing such an instance can be decomposed into three steps

\begin{enumerate}
    \item We start by setting $(\mathcal{X},\mathcal{Z})$ as a NP-hard problem, in this case, an instance of $\mathbf{IntHS}$.
    \item Based on $(\mathcal{X},\mathcal{Z})$, we construct a bijective mapping between $x,y$, to obtain a new NP-hard problem $(\mathcal{Y},\mathcal{Z})$ by preserving the $\mathbf{IntHS}$ structure.
    \item The bijective mapping should be designed carefully, such that the multimodal problem $(\mathcal{X},\mathcal{Y},\mathcal{Z})$ can be solved efficiently.
\end{enumerate}

Below we describe the construction of the instance and the main idea behind. A detailed proof is provided in the next section.

\textbf{Step 1:} We set one of the unimodal learning problem, say $(\mathcal{X},\mathcal{Z})$, as an instance of $\mathbf{IntHS}$. We denote any problem of $\mathbf{IntHS}$ by $H_1\cap H_2$ with halfspaces $H_1,H_2$ in $\mathbb{R}^n$, where each $H_i=(x|r_i^{\top}x\le c_i)$ is determined by the unit vector $r_i$ and $c_i\in \mathbb{R}$.

\textbf{Step 2:} A critical observation is that, any $\mathbf{IntHS}$ problem $H_1\cap H_2$ can be transformed into a new $\mathbf{IntHS}$ problem by applying a coordinate change, under which each $x$ is mapped to a new point with the corresponding $z$ remaining the same. Denote $Q\in \mathbb{R}^{n\times n}$ as any orthogonal matrix, we obtain $\hat{H}_1\cap \hat{H}_2$ where $\hat{H}_i=(x|\hat{r}_i^{\top}x\le c_i)$ by setting $\hat{r}_i=Q r_i$. Let $y=Qx$, we create a new NP-hard unimodal problem $(\mathcal{Y},\mathcal{Z})$, as $Q$ defines a bijective mapping from the set of all $\mathbf{IntHS}$ problems to itself. 

\textbf{Step 3:} It remains unclear how the multimodal problem $(\mathcal{X},\mathcal{Y},\mathcal{Z})$ can be easy to learn. Our strategy is to design a special $Q$ for each $H_1\cap H_2$, by encoding the information of $H_1\cap H_2$ into the transformation $Q$. Ideally, with $n$ linearly-independent $x_i$, we can recover the matrix $Q$ by basic linear algebra. With the exact values of $r_1,r_2$ in hand, we get $c_1,c_2$ by listing the distances from all $x$ to the hyperplane $r_i^{\top}x=0$ in $O(mn^2)$ time. The obtained classifier achieves zero loss on the training data.

However, it's challenging to directly encode the vectors $r_1,r_2$ into the $n\times n$ matrix $Q$. There are two main obstacles. First, how to encode the information of $r_1,r_2$ is unclear: $Q$ is under the constraint of an orthogonal matrix, which might be violated by simply filling $r_1,r_2$ into $Q$. Using more complicated techniques of encoding may bring other concerns such as the existence of a closed-form representation or whether decoding can be done efficiently. Second, the quality of such encoding is questionable: even if we find a way to encode $r_1,r_2$ into $Q$, we still need to make sure $(\mathcal{Y},\mathcal{Z})$ exhausts the set of all possible $\mathbf{IntHS}$ instances. Otherwise although each $(\mathcal{Y},\mathcal{Z})$ problem is an $\mathbf{IntHS}$ instance, the set of all possible $(\mathcal{Y},\mathcal{Z})$ problems is a merely a subset of $\mathbf{IntHS}$, preventing us from directly applying the NP-hardness result.

Fortunately, we have a very simple remedy: enlarging the dimension $n$ by twice, then using the first $n$ coordinates for $\mathbf{IntHS}$ while the latter $2n$ coordinates to encode the information of $\mathbf{IntHS}$. Roughly speaking, we create $2n$ null coordinates with no effect on the $\mathbf{IntHS}$ problem, while they carry the information of $\mathbf{IntHS}$ which can only be retrived by knowing both $x,y$. In particular, for any $\mathbf{IntHS}$ problem $H_1\cap H_2$, we set $Q$ as 
$$
Q=\begin{pmatrix}
I_n & 0 & 0\\
0 & \frac{r_1}{\sqrt{2}} &\cdots\\
0 & \frac{r_2}{\sqrt{2}} &\cdots
\end{pmatrix}.
$$
The vectors $r_1,r_2$ are simply flattened and set as the first column of the second block. Since the norm of this column is 1, $Q$ can be easily made feasible. The identity matrix $I_n$ ensures $(\mathcal{Y},\mathcal{Z})$ exhausts the set of all possible $\mathbf{IntHS}$ instances. The main result of this paper is given by the following theorem.

\begin{theorem}[Computational separation]\label{main}
    There exists a multimodal learning problem $(\mathcal{X},\mathcal{Y},\mathcal{Z})$ which is PAC-learnable in polynomial time, while both unimodal learning problems $(\mathcal{X},\mathcal{Z})$, $(\mathcal{Y},\mathcal{Z})$ are NP-hard, even if there is a bijective mapping $f: \mathcal{X}\to \mathcal{Y}$ such that $y=f(x), \forall (x,y,z)\sim (\mathcal{X},\mathcal{Y},\mathcal{Z})$.
\end{theorem}

Theorem \ref{main} demonstrates that multimodal learning solves some learning tasks exponentially faster than unimodal learning. Such exponential separation explains the empirical superiority of multimodal learning from the perspective of computation, supplementing the statistical guatantees in \cite{zhoul}.

Notably, the two pivotal factors leading to the statistical benefit of multimodal learning in \cite{zhoul}, namely connection and heterogeneity, are also evident in our construction. In particular, the mapping $Q$ between $\mathcal{X},\mathcal{Y}$ is bijective, meaning there exists a perfect connection between both modalities. On the other hand, $\mathcal{X},\mathcal{Y}$ carry different information about the problem, which is useless alone but effective when put together, indicating s strong heterogeneity.

\section{Proof of Theorem \ref{main}}
We first introduce the necessary ingredients for the construction of the learning problem. For each pair of unit vectors $v_1,v_2\in \mathbb{R}^n$, there exist orthogonal matrices in $\mathbb{R}^{2n}$ with its first column to be $(\frac{v_1}{\sqrt{2}},\frac{v_1}{\sqrt{2}})$ since $\|(\frac{v_1}{\sqrt{2}},\frac{v_1}{\sqrt{2}})\|_2=1$. In particular, for each pair $v_1,v_2$ we fix one such orthogonal matrix $F$, defining a function $F(v_1,v_2):\mathbb{R}^{2n}\to \mathbb{R}^{2n\times 2n}$ as below:
$$
F(v_1,v_2)=\begin{pmatrix}
\frac{v_1}{\sqrt{2}} & \cdots\\
\frac{v_2}{\sqrt{2}} & \cdots
\end{pmatrix}.
$$
In addition, we define an orthogonal transformation matrix $Q(v_1,v_2)\in \mathbb{R}^{3n\times 3n}$ as
$$
Q(v_1,v_2)=\begin{pmatrix}
I_n & 0\\
0 & F(v_1,v_2)
\end{pmatrix}.
$$
The matrix $Q(r_1,r_2)$ will serve as a fingerprint of an $\mathbf{IntHS}$ problem $H_1\cap H_2$. We also define a variant of the intersection of two half-spaces problem.

\begin{definition}[Low-dimensional intersection of two half-spaces]\label{defi}
    An instance of $\mathbf{IntHS_{\lambda}}$ is a set of points in $\mathbb{R}^n$ each labeled either ‘+’ or ‘-’, in which the labels only depend on the first $\lambda n$ coordinates where $\lambda\in (0,1)$ is a constant. The goal is to find an intersection of two half-spaces which correctly classifies the maximum number of points, where a ‘+’ point is classified correctly if it lies inside the intersection and a ‘-’ point is classified correctly if it lies outside of it.
\end{definition}

\begin{lemma}\label{lem NP}
    For every constant $\lambda>0$, learning $\mathbf{IntHS_{\lambda}}$ is NP-hard.
\end{lemma}

\begin{proof}
    We prove by reduction. Suppose for contradiction $\mathbf{IntHS_{\lambda}}$ can be learnt in polynomial time, then for each instance of $\mathbf{IntHS}$, we can create a new instance of $\mathbf{IntHS_{\lambda}}$ with dimension $\frac{n}{\lambda}$ by extension. In particular, each point $x\in \mathbb{R}^{\frac{n}{\lambda}}$ shares the same label as $x_{[1:n]}$ in the original $\mathbf{IntHS}$ instance. As a result, any classifier of $\mathbf{IntHS_{\lambda}}$ applies to the $\mathbf{IntHS}$ problem with the same accuracy, contradicting Proposition \ref{prop ks}.
\end{proof}

Now we are ready to state the learning problem $(\mathcal{X},\mathcal{Y},\mathcal{Z})$: $m$ data points $(x_i,y_i,z_i)$ are given, where $x_i,y_i\in \mathbb{R}^{3n}$ represent the two modalities and $z_i=\pm$ is the label. It's guaranteed that there is an intersection of two half-spaces that classifies all the points correctly, with supports of the defining unit vectors being the first $n$ coordinates. In other words, it's a realizable instance of $\mathbf{IntHS_{\frac{1}{3}}}$.

In particular, there are unit vectors $r_1,r_2\in \mathbb{R}^n$ and constants $c_1,c_2\in \mathbb{R}$ (unknown to the learner), such that all pairs $(x_i,z_i)$ can be perfectly classified by $\hat{H}_1\cap \hat{H}_2$, where $\hat{H}_i=(x|\hat{r}_i^{\top}x\le c_i)$ and $\hat{r}_i=(r_i,\mathbf{0}_{2n})$. Meanwhile, $y_i=Q(r_1,r_2)x_i$ holds for all data points, and all pairs $(y_i,z_i)$ can be perfectly classified by $\tilde{H}_1\cap \tilde{H}_2$, where $\tilde{H}_i=(x|\tilde{r}_i^{\top}x\le c_i)$ and $\tilde{r}_i=Q(r_1,r_2)(r_i,\mathbf{0}_{2n})$. 

Define the hypothesis set $\mathcal{S}$ as
$$
\mathcal{S}=\{h|h(x)=\textbf{sgn}(\min(c_1-r_1^{\top}x,c_2-r_2^{\top}x)), c_i\in \mathbb{R}, \|r_i\|_2=1\},
$$
which is exactly the set of all intersection of two half-spaces. We have the following results.

\begin{lemma} \label{lem xz}
    Properly learning $(\mathcal{X},\mathcal{Z})$ with $\mathcal{S}$ is NP-hard.
\end{lemma}

\begin{proof}
    It is a direct consequence of Lemma \ref{lem NP}, noticing that $(\mathcal{X},\mathcal{Z})$ is an $\mathbf{IntHS_{\frac{1}{3}}}$ instance.
\end{proof}

\begin{lemma} \label{lem yz}
    Properly learning $(\mathcal{Y},\mathcal{Z})$ with $\mathcal{S}$ is NP-hard.
\end{lemma}

\begin{proof}
    Although $(\mathcal{Y},\mathcal{Z})$ is also an $\mathbf{IntHS_{\frac{1}{3}}}$ instance, we still need to verify that $(\mathcal{Y},\mathcal{Z})$ exhausts all possible $\mathbf{IntHS_{\frac{1}{3}}}$ instances (otherwise we can't apply Lemma \ref{lem NP}, for example when all $(\mathcal{Y},\mathcal{Z})$ obey the same $\mathbf{IntHS_{\frac{1}{3}}}$ instance). Notice that $Q$ induces a mapping $H_1\cap H_2\to H_1\cap H_2$, and it's equivalent to proving it is a surjective mapping. For any $\mathbf{IntHS_{\frac{1}{3}}}$ instance $\hat{H}_1\cap \hat{H}_2$ where $\hat{H}_i=(x|\hat{r}_i^{\top}x\le c_i)$ and $\hat{r}_i=(r_i,\mathbf{0}_{2n})$, because $\hat{r}_i$ also has support in the first $n$ coordinates, we have that $ \hat{r}_i=Q(r_1,r_2)r_i$ with $r_i=\hat{r}_i$, proving the mapping is surjective.
\end{proof} 

\begin{lemma} \label{lem xyz}
    Assume $m\ge 3n$, $(\mathcal{X},\mathcal{Y},\mathcal{Z})$ is properly learnable with $\mathcal{S}$ (applied to $x$ only) in $O(mn^2)$ time, when there exist $3n$ data points with linearly-independent $x_i$.
\end{lemma}

\begin{proof}
    Consider the simple algorithm \ref{alg} which consists three steps: 
    \begin{enumerate}
        \item find a set $S$ of linearly-independent $x_i$ (line 2-6).
        \item find $Q$ by solving a linear system of $S$ (line 7-8).
        \item rank $x_i$ along the directions of $r_1,r_2$ to get $c_1,c_2$ (line 9-10).
    \end{enumerate}
    Step 1 runs in $O(mn^2)$ time, since testing orthogonality between two points runs in $O(n)$ time and $|S|=O(n)$. Step 2 runs in $O(n^3)$ time which is the complexity of solving a system of linear equations. Step 3 runs in $O(mn)$ time. Under our assumption $m\ge 3n$, the total running time is $O(mn^2+n^3+mn)=O(mn^2)$.

    We still need to verify the found classifier $h(x)$:
    $$
    h(x)=\textbf{sgn}(\min(c_1-r_1^{\top}x,c_2-r_2^{\top}x))
    $$
    does classify all data points correctly. By the construction of $Q$, we know there is a classifier $h^*(x)$ which classifies all data points correctly, which shares the same $r_i$ with $h(x)$:
    $$
    h^*(x)=\textbf{sgn}(\min(c^*_1-r_1^{\top}x,c^*_2-r_2^{\top}x)).
    $$
    By the choice of $c_1,c_2$, we have that $c_1\le c^*_1, c_2\le c^*_2$. Denote $h_{+}=\{x\in \mathbb{R}^{3n}, h(x)=+\}$, we have that
    $$
    (h_{+}\cap X) \subset (h^*_{+}\cap X)=X_+,
    $$
    by the fact $h_{+}\subset h^*_{+}$. Meanwhile, by the construction of $h(x)$, we have that $X_{+}\subset h_{+}$, and further
    $$
    X_{+}=(X_{+}\cap X)\subset (h_{+}\cap X).
    $$
    As a result, $X_{+}=h_{+}\cap X$ which means $h(x)$ does classify all data points correctly.
\end{proof}

\begin{algorithm}
\caption{Learning by decoding}
\label{alg}
\begin{algorithmic}[1]
\STATE Input: $m$ data points $(x_i,y_i,z_i)$. 
\STATE Set $S=\{x_1\}$, $t=2$.
\WHILE{$|S|<3n$}
\STATE If $x_t$ is orthogonal to each member of $S$, add $x_t$ to $S$.
\STATE $t=t+1$.
\ENDWHILE
\STATE Solving the linear system $Qx_i=y_i$, $\forall x_i\in S$.
\STATE Recover $r_1,r_2$ from $Q$.
\STATE Let $X_{+}$ be the set of all $x_i$ with $z_i=+$.
\STATE Set $c_i=\max_{x\in X_{+}} r_i^{\top}x$.
\end{algorithmic}
\end{algorithm}

Lemma \ref{lem xyz} concerns only the learnability on the training data, to extend this result to PAC-learnability we introduce the following definition.

\begin{definition}
    A data distribution $D$ on $(\mathcal{X},\mathcal{Y},\mathcal{Z})$ is called non-degenerate, if 
    $$
    \mathbb{P}_{(x_i,y_i,z_i)\sim D, i\in [3n]}(\exists \lambda\ne \mathbf{0}, s.t. \sum_{i=1}^{3n} \lambda_i x_i=0)=0.
    $$
\end{definition}
Most distributions whose support has non-zero measure are non-degenerate, including common uniform and Gaussian distributions. We have the following result for PAC-learnability.

\begin{lemma} \label{lem xyzpac}
    Assume $m$ data points are sampled from a non-degenerate distribution $D$ and $m\ge 3n$, $(\mathcal{X},\mathcal{Y},\mathcal{Z})$ is properly PAC-learnable with $\mathcal{S}$ (applied to $x$ only) in $O(mn^2)$ time. In particular, with probability at least $1-\delta$, the generalization error $\epsilon$ of algorithm \ref{alg} is upper bounded by 
    $$
    \epsilon=O\left(\sqrt{\frac{n\log m+\log \frac{1}{\delta}}{m}}\right).
    $$
\end{lemma}

\begin{proof}
    By the assumption that $D$ is non-degenerate, we have that with probability 1, there exist $3n$ data points with linearly-independent $x_i$. By the conclusion of Lemma \ref{lem xyz}, the learnt classifier achieves zero loss on training data.

    From classic statistical learning theory, the generalization error of such classifier can be characterized by the VC-dimension of the hypothesis class.

    \begin{theorem}[\cite{vapnik2015uniform}]
        With probability at least $1-\delta$, for every $h$ in the hypothesis class $\mathcal{H}$, if $h$ is consistent with $m$ training samples, the generalization error $\epsilon$ of $h$ is upper bounded by 
    $$
    \epsilon=O\left(\sqrt{\frac{d\log m+\log \frac{1}{\delta}}{m}}\right),
    $$
    where $d$ denotes the VC-dimension of $\mathcal{H}$.
    \end{theorem}
    We only need to determine the VC-dimension of the class of intersection of two half-spaces in $\mathbb{R}^{3n}$. It's well known the VC-dimension of a single half-space is $O(n)$. \cite{blumer1989learnability} shows that the $k$-fold intersection of any VC-class has VC-dimension bounded by $O(d k \log k)$. Putting $d=n$ and $k=2$ concludes the proof.
\end{proof}

\section{Separation in Improper Learning}
As an extension of our result in proper learning, we consider the problem whether multimodality still possesses such exponential computational benefit when the learner is allowed to output arbitrary hypothesis beyond the hypothesis set $\mathcal{H}$, i.e. the improper learning setting.

The general problem of learning intersections of halfspaces is known to be hard even in the improper learning setting, defined as below.

\begin{definition}[Intersection of half-spaces]
    An instance of $\mathbf{IntHS(N)}$ is a set of points in $\mathbb{R}^n$ each labeled either ‘+’ or ‘-’ and the goal is to find an intersection of $N$ number of half-spaces which correctly classifies the maximum number of points, where a ‘+’ point is classified correctly if it lies inside the intersection and a ‘-’ point is classified correctly if it lies outside of it.
\end{definition}

We will rely on the following hardness of improper learning intersections of halfspaces.

\begin{theorem}\label{daniel}[\cite{daniely2014average,daniely2021local}]
    If $\lim_{n\to \infty} q(n)=\infty$ is a super-constant, there is no efficient algorithm that improperly learns $q(n)$ numbers of intersections of halfspaces in $R^n$.
\end{theorem}

Using a similar analysis to Theorem \ref{main}, we obtain the following separation result.

\begin{theorem}[Improper computational separation]\label{main im}
    There exists a multimodal learning problem $(\mathcal{X},\mathcal{Y},\mathcal{Z})$ which is PAC-learnable in polynomial time, while both unimodal learning problems $(\mathcal{X},\mathcal{Z})$, $(\mathcal{Y},\mathcal{Z})$ are NP-hard in the improper learning setting, even if there is a bijective mapping $f: \mathcal{X}\to \mathcal{Y}$ such that $y=f(x), \forall (x,y,z)\sim (\mathcal{X},\mathcal{Y},\mathcal{Z})$.
\end{theorem}

\begin{proof}
    The proof is identical to that of Theorem \ref{main} except for two minor places:

    \begin{enumerate}
        \item We need a strengthened version of Lemma \ref{lem NP} with $\lambda$ being $1/\textbf{poly}(n)$.
        \item The hard instance construction and the algorithm of multimodal learning is slightly modified to accommodate the new Lemma.
    \end{enumerate}
    
    We begin with the strengthened version of Lemma \ref{lem NP}. The definition of $\mathbf{IntHS}(N)_{\lambda}$ follows directly from Definition \ref{defi}, and we won't repeat here.

    \begin{lemma}\label{lem NP new}
    Given any super-constant $q(n)$. For all constants $C\ge 1,c>0$, improperly learning $\mathbf{IntHS}(q(n))_{\frac{1}{C n^c}}$ is NP-hard.
    \end{lemma}

\begin{proof}
    We prove by reduction. Suppose for contradiction $\mathbf{IntHS}(q(n))_{\frac{1}{C n^c}}$ can be learnt in polynomial time, then for each instance of $\mathbf{IntHS}(q(n))$, we create a new instance of $\mathbf{IntHS}(q'(C n^{c+1}))_{\frac{1}{C n^c}}$ with dimension $C n^{c+1}$ by extension, where $q'(C n^{c+1})=q(n)$ is still a super-constant. In particular, each point $x\in \mathbb{R}^{C n^{c+1}}$ shares the same label as $x_{[1:n]}$ in the original $\mathbf{IntHS}(q(n))$ instance. Since and polynomial of $Cn^c$ is also a polynomial of $n$, we conclude that any classifier of $\mathbf{IntHS}(q(n))_{\frac{1}{C n^c}}$ applies to the $\mathbf{IntHS}(q(n))$ problem with the same accuracy, contradicting Theorem \ref{daniel}.
\end{proof}

Specifically, we will only use Lemma \ref{lem NP new} with $C=1, c=1/2$. Now we are ready to state the learning problem $(\mathcal{X},\mathcal{Y},\mathcal{Z})$: $m$ data points $(x_i,y_i,z_i)$ are given, where $x_i,y_i\in \mathbb{R}^{n}$ represent the two modalities and $z_i=\pm$ is the label. It's guaranteed that there is an intersection of $\sqrt{n}-1$ number of half-spaces $H_1,H_2,\cdots,H_{\sqrt{n}-1}$ that classifies all the points correctly, with supports of the defining unit vectors being the first $\sqrt{n}$ coordinates. In other words, it's a realizable instance of $\mathbf{IntHS}(\sqrt{n}-1)_{1/\sqrt{n}}$ (with $q(x)=x-1$).

In particular, there are unit vectors $r_1,r_2,\cdots,r_{\sqrt{n}-1}\in\mathbb{R}^{\sqrt{n}}$ and constants $c_1,c_2,\cdots,c_{\sqrt{n}-1}\in \mathbb{R}$ (unknown to the learner), such that all pairs $(x_i,z_i)$ can be perfectly classified by $\cap_i \hat{H}_i $, where $\hat{H}_i=(x|\hat{r}_i^{\top}x\le c_i)$ and $\hat{r}_i=(r_i,\mathbf{0}_{n-\sqrt{n}})$. Similarly, we can define the $Q$ matrix as
$$
Q(v_1,v_2,\cdots,v_{\sqrt{n}-1})=\begin{pmatrix}
I_{\sqrt{n}} & 0\\
0 & F(v_1,v_2,\cdots,v_{\sqrt{n}-1})
\end{pmatrix},
$$
where the function $F(v_1,v_2,\cdots,v_{\sqrt{n}-1}):\mathbb{R}^{n-\sqrt{n}}\to \mathbb{R}^{(n-\sqrt{n})\times (n-\sqrt{n})}$ is chosen as below:
$$
F(v_1,v_2,\cdots,v_{\sqrt{n}-1})=\begin{pmatrix}
\frac{v_1}{\sqrt{\sqrt{n}-1}} & \cdots\\
\frac{v_2}{\sqrt{\sqrt{n}-1}} & \cdots\\
\cdots & \cdots\\
\frac{v_{\sqrt{n}-1}}{\sqrt{\sqrt{n}-1}} & \cdots
\end{pmatrix}.
$$

Meanwhile, $y_i=Q(v_1,v_2,\cdots,v_{\sqrt{n}-1})x_i$ holds for all data points, and all pairs $(y_i,z_i)$ can be perfectly classified by $\cap_i \tilde{H}_i$, where $\tilde{H}_i=(x|\tilde{r}_i^{\top}x\le c_i)$ and $\tilde{r}_i=Q(v_1,v_2,\cdots,v_{\sqrt{n}-1})(r_i,\mathbf{0}_{n-\sqrt{n}})$. 

Via the same argument as Theorem \ref{main}, according to Lemma \ref{lem NP new}, both improperly learning $(\mathcal{X},\mathcal{Z})$ and improperly learning $(\mathcal{Y},\mathcal{Z})$ are hard. 

We only need to show $(\mathcal{X},\mathcal{Y},\mathcal{Z})$ can be learnt efficiently. The same algorithm will be applied to decode all the $r_i$ and $c_i$ is set as $\max_{x\in X_{+}} r_i^{\top}x$. The classifier we use is still
$$
    h(x)=\textbf{sgn}(\min(c_1-r_1^{\top}x,c_2-r_2^{\top}x,\cdots,c_{\sqrt{n}-1}-r_{\sqrt{n}-1}^{\top}x)).
$$
 By the construction of $Q$, we know there is a classifier $h^*(x)$ which classifies all data points correctly, which shares the same $r_i$ with $h(x)$:
    $$
    h^*(x)=\textbf{sgn}(\min(c^*_1-r_1^{\top}x,c^*_2-r_2^{\top}x,\cdots,c^*_{\sqrt{n}-1}-r_{\sqrt{n}-1}^{\top}x)).
    $$
    By the choice of $c_i$, we have that $c_i\le c^*_i, \forall i$. Denote $h_{+}=\{x\in \mathbb{R}^{n}, h(x)=+\}$, we have that
    $$
    (h_{+}\cap X) \subset (h^*_{+}\cap X)=X_+,
    $$
    by the fact $h_{+}\subset h^*_{+}$. Meanwhile, by the construction of $h(x)$, we have that $X_{+}\subset h_{+}$, and further
    $$
    X_{+}=(X_{+}\cap X)\subset (h_{+}\cap X).
    $$
    As a result, $X_{+}=h_{+}\cap X$ which means $h(x)$ does classify all data points correctly. The rest of the proof on PAC learning easily follows from Theorem \ref{main} and we omit it here.

\end{proof}

\section{Conclusion}
In this paper, we take a preliminary step towards unraveling the computational benefit of multimodal learning. We demonstrate an exponential separation in computation between multimodal and unimodal learning by constructing a variant of the intersection of two half-spaces
problem, which is NP-hard for any unimodal algorithm but can be efficiently solved by a multimodal algorithm. Complementing the statistical merits of multimodal learning as shown in \cite{zhoul}, our result provides a more comprehensive theoretical understanding of the power of multimodal learning.

The main limitation of this work, in our opinion, is on the contrived argument that multimodal learning is tractable: the efficient learning scheme provided in this work only succeeds on a narrow, intricately designed class of problem instances. These results alone are not enough to suggest that computational benefits of multimodal learning are a more general phenomenon. 

We conclude with two questions as future directions to improve the preliminary results presented in this work. 

\begin{enumerate}
    \item Can we show such separation in computation for more natural learning problems? Ideally, a good efficient learning algorithm for the multimodal setting should have less dependence on the problem structure, such as ERM.
    \item Can we obtain a general sufficient condition for the computational benefit of multimodal learning? Even a polynomial improvement is interesting.
\end{enumerate}

\bibliography{Xbib}
\bibliographystyle{plain}

\end{document}